\documentclass[letterpaper, 10pt, conference]{ieeeconf} 
\IEEEoverridecommandlockouts
\usepackage{graphicx}
 \usepackage{resizegather}
\usepackage{url}
\usepackage{xcolor}
\usepackage{amsmath}
\usepackage{mathtools}
 \mathtoolsset{showonlyrefs}
\usepackage{tikz}
\usepackage{verbatim}
\usepackage{subcaption}
\usepackage{varwidth}
\usepackage{pgfplots}
\pgfplotsset{compat=1.10}
\usepackage{graphicx}
\usepackage{epstopdf}
\usepackage{amssymb}
\usepackage{relsize}
\usepackage[textwidth=3.8em,textsize=scriptsize,disable]{todonotes}
\usepackage{algorithm}
\definecolor{Gray}{gray}{0.90}
\usepackage{colortbl}

\usepackage{cite}
\usepackage{algorithm}
\usepackage{algorithmicx}
\usepackage{algcompatible}
\usepackage{algpseudocode}
\usepackage{color}
\usepackage{blkarray}
\usepackage{setspace} 
\usepackage{mathtools}
\usepackage{upgreek}
\usepackage{multirow}
\usepackage{caption}
\usepackage{subcaption}
\usepackage{hyperref}
\usepackage{arydshln}

\newtheorem{theorem}{\bf{Theorem}}[section]

\newtheorem{prop}[theorem]{\bf{Proposition}}

\newtheorem{lemma}{Lemma}

\newenvironment{definition}[1][Definition]{\begin{trivlist}
\item[\hskip \labelsep {\bfseries #1}]}{\end{trivlist}}

\usepackage{amssymb}
\usepackage{url}

\newcommand{\calT}[0]{\mathcal{T}}

\newcommand{\calH}[0]{\mathcal{H}}
\newcommand{\calD}[0]{\mathcal{D}}
\newcommand{\calM}[0]{\mathcal{M}}

\usepackage{scalerel}

\newcommand\fat[1]{\ThisStyle{\ooalign{%
  \kern.46pt$\SavedStyle#1$\cr\kern.33pt$\SavedStyle#1$\cr%
  \kern.2pt$\SavedStyle#1$\cr$\SavedStyle#1$}}}
\definecolor{amber}{rgb}{1.0, 0.75, 0.0}

\newtheorem{prob}{\textbf{Problem}}

\usepackage[capitalize,noabbrev]{cleveref}

\newtheorem{probs}{\textbf{Problem}}
\numberwithin{probs}{section} 

\begin{document}

\title{Control-based Graph Embeddings with Data Augmentation\\ for Contrastive Learning}
\author{
Obaid Ullah Ahmad, Anwar Said, Mudassir Shabbir, Waseem Abbas, Xenofon Koutsoukos
\thanks{Obaid Ullah Ahmad is with the Electrical Engineering Department at the University of Texas at Dallas, Richardson, TX. Email: Obaidullah.Ahmad@utdallas.edu.}
\thanks{Anwar Said, Mudassir~Shabbir and Xenofon Koutsoukos are with the Computer Science Department at the Vanderbilt University, Nashville, TN. Emails: anwar.said,mudassir.shabbir, xenofon.koutsoukos@vanderbilt.edu.} 
\thanks{Waseem~Abbas is with the Systems Engineering Department at the University of Texas at Dallas, Richardson, TX. Email: waseem.abbas@utdallas.edu}
}

\maketitle
\begin{abstract}
In this paper, we study the problem of unsupervised graph representation learning by harnessing the control properties of dynamical networks defined on graphs. Our approach introduces a novel framework for contrastive learning, a widely prevalent technique for unsupervised representation learning. A crucial step in contrastive learning is the creation of `augmented' graphs from the input graphs. Though different from the original graphs, these augmented graphs retain the original graph's structural characteristics. Here, we propose a unique method for generating these augmented graphs by leveraging the control properties of networks. The core concept revolves around perturbing the original graph to create a new one while preserving the controllability properties specific to networks and graphs. Compared to the existing methods, we demonstrate that this innovative approach enhances the effectiveness of contrastive learning frameworks, leading to superior results regarding the accuracy of the classification tasks. The key innovation lies in our ability to decode the network structure using these control properties, opening new avenues for unsupervised graph representation learning.
\end{abstract}


\section{Introduction}

Networks serve as fundamental data structures for representing relationships, connectivity, and interactions across various domains, such as social networks, biology, transportation, brain connectivity, and recommendation systems~\cite{newman2018networks
, said2023neurograph}. Network representation learning plays a pivotal role in acquiring meaningful network representations, which find applications in tasks like node classification, link prediction, and community detection~\cite{said2021dgsd}. Traditional network representation learning heavily relies on supervised learning, necessitating substantial labeled data for effective training~\cite{hamilton2017inductive}. However, obtaining labeled network data is often challenging, expensive, and limited in availability.

Contrarily, contrastive learning (CL) has emerged as a prominent self-supervised learning (SSL) technique in unsupervised network representation learning~\cite{oord2018representation}. CL methods operate by comparing augmented positive and negative samples with the original graph. The positive samples exhibit similarity, while the negative samples manifest dissimilarity. This framework empowers CL methods to acquire representations that capture the inherent network structure, even when labeled data is absent~\cite{chen2020simple}. Graph Contrastive Representation Learning has recently gained attention in the context of graph representation learning, aiming to maximize agreement between similar subgraphs and produce informative embeddings that capture the graph structure~\cite{you2020graph}. While existing GCRL approaches primarily focus on node-level embeddings~\cite{xia2022hypergraph}, our proposed architecture has the potential to generate graph-level embeddings suitable for SSL.

In this work, we introduce a novel approach that leverages control properties to design graph-level embeddings for self-supervised learning. Recent research has uncovered deep connections between network controllability and various graph-theoretic constructs, including matching, graph distances, and zero forcing sets~\cite{yaziciouglu2016graph, monshizadeh2014zero,mesbahi2010graph}. Additionally, significant progress has been made in characterizing the controllability of different families of network graphs, including paths, cycles, random graphs, circulant graphs, and product graphs~\cite{mesbahi2010graph}. These investigations shed light on the interplay between network structures and their controllability properties, enhancing our understanding of network dynamics. Our aim is to explore and harness the interconnections between network structures and their controllability properties to form a foundation for comprehensive graph representations.

Furthermore, we introduce systematic graph augmentation for creating positive and negative pairs in CL, in contrast to previous random edge perturbation methods~\cite{you2020graph}. Our systematic approach focuses on preserving the graph's control properties, leading to improved performance in downstream machine-learning tasks.

Our main contributions can be summarized as follows:

\begin{itemize}
    \item  We introduce a novel graph embedding---representing graphs as vectors--- called CTRL, which is based on the control properties of networks defined on graphs, including meaningful metrics of controllability such as the spectrum of the Gramian matrix.
    \item We present the Control-based Graph Contrastive Learning architecture for unsupervised representation learning of networks, applicable to various downstream graph-level tasks.
    \item We devise innovative augmentation techniques that mainly preserve the controllability of the network.
    \item We conduct extensive numerical evaluations on real-world graph datasets, showcasing the effectiveness of our method in graph classification compared to several state-of-the-art (SOTA) benchmark methods.
\end{itemize}

\section{Preliminaries and Problem Statement}

\subsection{Notations}
A network composed of interconnected entities is represented as a graph, denoted as $G=(V,E)$, where the set of vertices $V=V(G)={v_1,v_2,\ldots,v_N}$ represents the entities, and the collection of edges $E=E(G)\subseteq V\times V$ represents pairs of entities with established relationships. In this context, the terms 'vertex,' 'node,' and 'agent' are used interchangeably. The neighborhood of a vertex $v_i$ is defined as $\mathcal{N}_i = {v_j\in V: (v_j,v_i)\in E}$. 
The \emph{distance} between vertices $v_i$ and $v_j$, denoted as $d(v_i,v_j)$, represents the shortest path length between them.
The transpose of a matrix $X$ is denoted as $X^T$. An $N$-dimensional vector with all its entries set to zero is represented as $\boldsymbol{0}_N$, and a vector with all its entries set to 1 is denoted as $\boldsymbol{1}_N$.
Although our primary focus is on undirected graphs for simplicity in explanation, it is important to note that all methods and findings are equally applicable to directed graphs.

\subsection{Problem Description}
\label{sec:problem_desc}

In this subsection, we introduce the problem of generating graph-level representations in an unsupervised manner. We begin by providing a formal definition of the problem and then present a contrastive learning-based solution. For graph-level embeddings, we develop a control-based graph embedding method in Section \ref{sec:control-section}

A graph embedding, denoted as $\phi(G):\mathcal{G}\rightarrow \mathbb{R}^d$, maps graphs from the family $\mathcal{G}$ to a Euclidean space of dimension $d$. The main goal of graph embedding is to meet specific design criteria. 
Due to the limited availability of labeled data for large real-world benchmark datasets, we need a mechanism to learn graph representations without heavy reliance on labels.
One crucial objective is to ensure that $\phi$ retains information about structural similarities between pairs of graphs, both at local and global scales. This means that if two graphs share structural similarities, their embeddings should yield vectors that are close to the target vector space, as measured by Euclidean distance. It's important to note that the notion of similarity between graphs depends on the specific application and the type of graphs being considered, such as chemical compounds or social networks.

Another crucial design objective for graph embeddings is scalability. An optimal graph embedding should not only map graphs of different sizes to a fixed-dimensional space but should also transcend graph size to capture underlying structural characteristics. For example, an effective graph embedding would position the mapping of a ten-node circulant graph closer to that of a twenty-node circulant graph compared to the mapping of a fifteen-node wheel graph. In this paper, we tackle the challenge of generating graph embeddings while adhering to these design objectives.
\vspace{1 mm}
\noindent\fbox{\begin{varwidth}{\dimexpr\linewidth-1\fboxsep-2\fboxrule\relax}
    \begin{prob}
    \label{prob:CL}
        Given a graph $G$, generate unsupervised graph representations $\phi(G)$ that capture essential structural information and node relationships for subsequent machine learning tasks.
    \end{prob}
\end{varwidth}}

These learned representations $\phi(G)$ are intended to be semantically meaningful and effective for various downstream tasks, including node classification, link prediction, graph classification, and community detection.

\subsection{Proposed Approach}
A typical approach to learning unsupervised representation of raw data is called self-supervised learning (SSL). One of the powerful techniques of SSL is contrastive learning which has achieved remarkable success across various domains, including computer vision\cite{chen2020simple}. In the realm of graph representation learning, researchers have introduced Contrastive Graph Representation Learning (CGRL) \cite{duan2023self}. This approach operates on the principle of generating diverse augmented perspectives of the same data samples through pretext tasks. We propose to introduce a dynamical system on graphs, examining the control characteristics of this system, and subsequently crafting an embedding that utilizes these control attributes, as elaborated in Section \ref{sec:control-section}.



\paragraph{Graph Contrastive Representation Learning} Graph Contrastive Representation Learning (GCRL) offers distinct advantages over traditional unsupervised graph representation methods. GCRL encourages the model to bring similar nodes or subgraphs closer together in the embedding space while pushing dissimilar nodes or subgraphs farther apart \cite{you2020graph}, enhancing performance in various downstream tasks \cite{sun2019infograph}. It excels in data efficiency, leveraging limited labeled data efficiently with unlabeled data sources \cite{you2020graph}. GCRL is scalable, handling large-scale graph datasets effectively \cite{you2020graph, hassani2020contrastive}. It facilitates easy transfer of learned representations to diverse downstream tasks, including node classification, link prediction, and graph classification \cite{you2020graph}. Lastly, GCRL often produces interpretable embeddings, aiding in the understanding and analysis of the learned representations \cite{sun2019infograph}.

\paragraph{Control-based Graph Contrastive Learning (CGCL)} 
We introduce Control-based Graph Contrastive Learning (CGCL), which computes control-based graph-level features denoted as $CTRL(.)$. These features are utilized to bring an augmented version $G^\prime$ of a graph $G$ closer together in a latent space $z(.)$ using the normalized temperature-scaled cross-entropy loss (NT-Xent) \cite{oord2018representation}. Additionally, we propose various augmentation techniques designed to preserve the $CTRL$ properties of the graph to a certain extent. This is illustrated in Figure \ref{fig:block_dia}.

For a given graph $G$, we apply a control-based augmentation $\calT(G)$ to obtain $G^\prime$, creating a positive pair. We then compute control-based features $CTRL(.)$ for both $G$ and $G^\prime$ and pass them through a learnable encoder $f(.)$ to transform them into a new latent space. The goal is to optimize the similarity of each positive pair in this latent space. This concept is illustrated in Figure \ref{fig:block_dia}, where the embeddings $z_G$ and $z_{G^\prime}$ are represented as yellow cuboids.

\begin{figure*}[t]
    \centering
    \includegraphics[scale = 0.4]{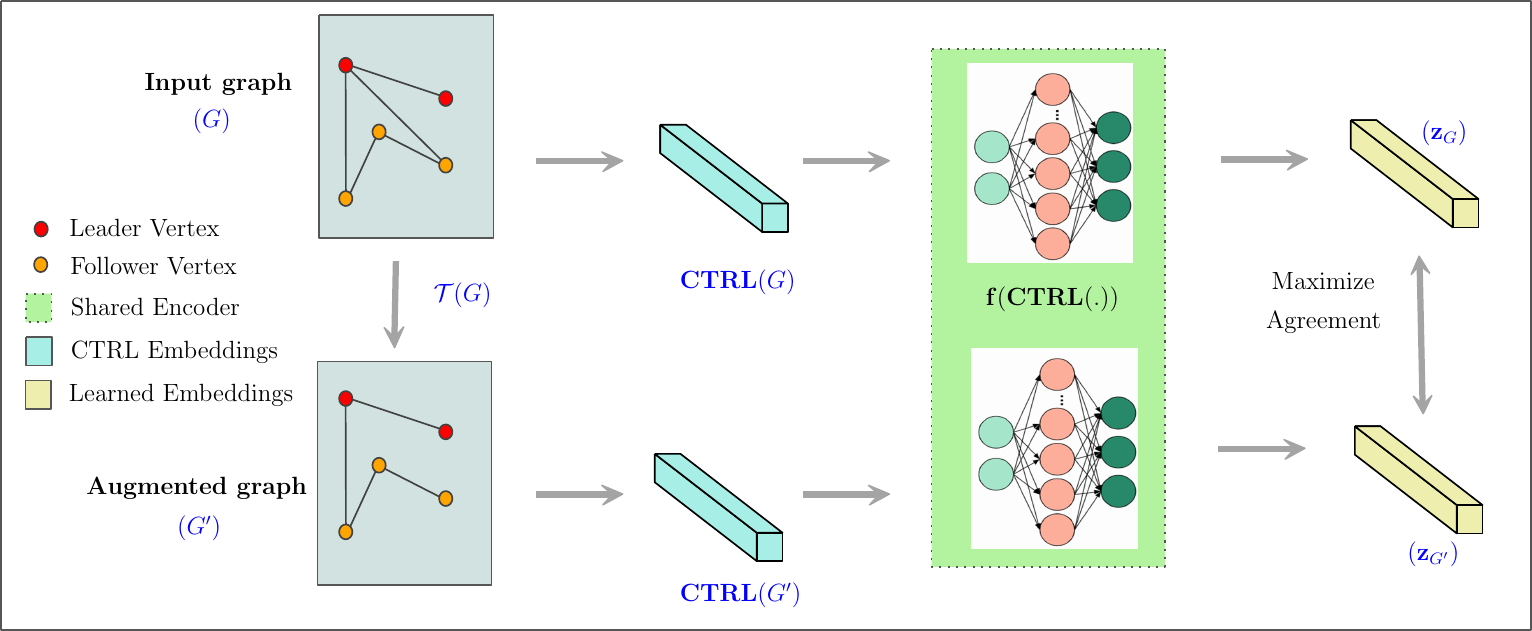}
    \caption{Block diagram of the proposed CGCL approach}
    \label{fig:block_dia}
\end{figure*}

This optimization is achieved using the NT-Xent loss proposed by Oord et al. \cite{oord2018representation}. The loss function encourages the similarity between the embeddings of the original graph and its transformed counterpart (positive pair) while minimizing the similarity with the transformed embeddings $z(.)$ of other graphs in the dataset (negative pairs). This self-supervised learning approach enables the model to capture meaningful representations effectively. The NT-Xent loss is employed as:
\begin{equation}    
\mathcal{L} =  \mathbb{E}\left[-log\frac{exp(sim(z_{G}, z_{G^\prime})/\tau)}{ \frac{1}{|\mathcal{G}|}\sum_{g \in \mathcal{G}, g \neq G} exp(sim(z_{G}, z_{g^\prime})/\tau)}\right],
\end{equation}
where $sim(z_{G}, z_{G^\prime})$ represents the cosine similarity between the embeddings of the graph $G$ and its augmentation $G^\prime$, $\mathcal{G}$ is the set containing all the graphs in the dataset, $g^\prime$ is the augmented version of graph $g$, and $\mathbb{E}$ is the expectation. The temperature hyperparameter $\tau$ is used to control the sharpness of the distribution.

In summary, CGCL leverages contrastive learning principles to generate expressive graph representations by considering control-based graph-level features and optimizing the similarity of positive pairs, demonstrating its potential for self-supervised graph representation learning.

\section{Network Controllability and Graph Embeddings}
\label{sec:control-section}
In this section, we introduce a novel graph embedding approach rooted in the control properties of networks. We begin by presenting a network as a controllable dynamic system. We then provide a formal definition of network controllability and subsequently explore several metrics employed to quantify it. These metrics serve as the foundation for constructing control-based graph embeddings, denoted as $CTRL(G)$, for a given graph $G$~\cite{said2023network}.

\subsection{Networks as Dynamical Systems}

In the context of network dynamics, each agent, denoted as \(v_i\), represents a dynamic unit with a \emph{state} \(x_i(t) \in \mathbb{R}\) at time \(t\). These agents share their states with neighboring agents in \(\mathcal{N}_i\) and update their states based on specific dynamics, like consensus dynamics. The collective system state at time \(t\) is represented as the vector \(x(t) = [x_1(t) \; x_2(t) \; \ldots \; x_N(t)]^T\).

To control the states of this dynamical system, we introduce external control signals applied to a subset of agents known as \emph{leaders}. These leader agents have states that can be directly manipulated, expressed as \(\dot{x}_l = u_l(t)\), where \(u_l(t)\) represents the input signal. Conversely, non-leader agents, referred to as \emph{followers}, update their states by aggregating information from their local neighborhoods. The dynamics of the followers, denoted as \(\dot{x}_f(t)\) in this leader-follower system, can be expressed as:
\[
\dot{x}_f(t) = -\calM(G) x_f(t) + \calH(G) u(t),
\]
where \(\calM(G)\) represents system matrices related to the followers' subgraph, and \(\calH(G)\) denotes the topological interactions between leader and follower agents. We use the Laplacian matrix for a matrix representation of a network.

For a network represented as a graph \(G = (V, E)\), we partition the node set \(V\) into two groups: followers (\(V_f\)) and leaders (\(V_\ell\)), where \(|V_f| = N_f\) and \(|V_\ell| = N_\ell\). We establish an ordered arrangement for the nodes, with \(V_f\) consisting of nodes \(v_1, v_2, \ldots, v_{N_f}\), and \(V_\ell\) containing nodes \(v_{N_f + 1}, \ldots, v_{N}\). The subgraph composed of follower nodes is referred to as the \emph{follower graph} (\(G_f\)), represented mathematically using the Laplacian matrix \(L\), which is partitioned as:
\begin{equation}
\label{eq:Laplacian_Partition}
    L =\left[
\begin{array}{c|c}
A & B \\
\hline
B^T & C
\end{array}
\right],
\end{equation}
where \(A\in \mathbb{R}^{N_f\times N_f}\), \(B\in\mathbb{R}^{N_f\times N_\ell}\), and \(C\in\mathbb{R}^{N_\ell \times N_\ell}\).

In this configuration, we introduce an external input signal \(u_l\) applied to leader agent \(v_l \in V_\ell\). The follower nodes update their states according to the dynamics given by:
\begin{equation}
    \label{eq:follower_dynamics}
    \dot{\boldsymbol{x}}_f(t) = -A\boldsymbol{x_f}(t) - B u(t), 
\end{equation}
where \(\boldsymbol{x_f}(t)\in\mathbb{R}^{N_f}\) represents the state vector of follower nodes at time \(t\), and \(u(t) = [u_{N_f+1}(t) \; \cdots \; u_{N}(t)]^T \in \mathbb{R}^{N_\ell}\) is the control signal at time \(t\). The matrices \(-A\) and \(-B\) in these dynamics are derived from the network structure and leader agent selection.

From a control perspective, we are interested in assessing the feasibility of steering the system described by these dynamics from an initial state to a final state within a finite time interval \(t_1\). If control is achievable, we aim to quantify the control energy \(\mathcal{E}(u)\) required, as defined below:
\begin{equation}
\label{eq:control_energy}
\mathcal{E}(u) = \int_{\tau = 0}^{t_1} \|u(\tau)\|^2 d\tau.
\end{equation}
We also investigate the dimension of the subspace containing reachable states and the impact of changing leader agents. These inquiries provide insights into the underlying graph structure and guide the derivation of network controllability metrics for graph embeddings.

\subsection{Network Controllability Metrics}
\label{control-matrics}
The process of controlling a network entails the responsibility of directing it from an initial state to a desired final state by applying control inputs to specific leader nodes within the network.
A state $\boldsymbol{x_f^\ast}\in\mathbb{R}^{N_f}$ is considered \emph{reachable} when there exists an input that can propel the network from the origin $\boldsymbol{0}_{N_f}$ to the state $\boldsymbol{x_f^\ast}$ within a finite timeframe. The set comprising all such reachable states defines what we refer to as the \emph{controllable subspace}. Importantly, in continuous linear time-invariant systems, such as the one described by equation \eqref{eq:follower_dynamics}, if a state $\boldsymbol{x_f}^\ast$ is reachable from the origin, it is also reachable from any arbitrary initial state within any given duration of time.

The \emph{dimension} of this controllable subspace $\gamma(G, V_\ell)$, is a pivotal concept in control theory. It can be determined by examining the \emph{rank} of the \emph{Controllability matrix}: $\mathcal{C} = \left[ \begin{array}{llllll}
-B & (-A)(-B) & \cdots & (-A)^{N_f-1}(-B)
\end{array}\right]$. The rank of this matrix hinges on the properties of matrices $A$ and $B$, which, in turn, are contingent on the network's structure and the selection of leader nodes. 

The \emph{Controllability Gramian} serves as a significant mathematical construct that offers vital insights into the control characteristics of a network \cite{pasqualetti2014controllability,summers2015submodularity,wu2018benchmarking}. Utilizing the Controllability Gramian, we can quantitatively assess the ease of transitioning from one state to another, taking into account the necessary control energy as defined in equation \eqref{eq:control_energy}.

For the system delineated in equation \eqref{eq:follower_dynamics}, the \emph{infinite horizon controllability Gramian} is defined as follows:
\begin{equation}
\label{eq:Gramian}
\mathcal{W} = \int_{0}^{\infty}e^{-A\tau}(-B)(-B)^Te^{-A^T\tau}d\tau \; \in \; \mathbb{R}^{N_f \times N_f}.
\end{equation}

If the system is stable, signifying that all eigenvalues of $-A$ have negative real parts, $\mathcal{W}$ asymptotically converges and can be computed through the Lyapunov equation: \begin{equation}
\label{eq:Lyapunov}
(-A)\mathcal{W} + \mathcal{W}(-A)^T + (-B)(-B)^T  = 0,
\end{equation}

For a solution to exist for \eqref{eq:Lyapunov}, it is necessary for $-A$ to be a stable matrix. This condition holds for connected graphs.

\begin{lemma}
If we partition the Laplacian matrix $L$ of an undirected connected graph as shown in \eqref{eq:Laplacian_Partition}, the matrix $A$ is positive definite \cite{mesbahi2010graph}. \end{lemma}

In summary, when partitioning the Laplacian matrix $L$ of an undirected connected graph, as demonstrated in equation \eqref{eq:Laplacian_Partition}, the matrix $A$ is revealed to be positive definite, ensuring the system's stability. This stability enables the computation of the Controllability Gramian $\mathcal{W}$, which serves as a valuable measure of controllability in terms of energy-related quantification. It also facilitates the derivation of various controllability statistics \cite{pasqualetti2014controllability,summers2015submodularity,wu2018benchmarking}. Some of these statistics are further discussed below.
\begin{itemize}
\item[i] {\textbf{Trace of $\mathcal{W}$}}: The trace of the controllability Gramian inversely correlates with the average control energy required to reach random target states. It also indicates the overall controllability across all directions within the state space.

\item[ii] \textbf{Minimum eigenvalue of $\mathcal{W}$}: This metric represents the worst-case scenario and demonstrates an inverse relationship with the control energy required to navigate the network in the least controllable direction.

\item[iii] \textbf{Rank of $\mathcal{W}$}: The rank of $\mathcal{W}$ corresponds to the dimension of the controllable subspace.

\item[iv] \textbf{Determinant of $\mathcal{W}$}: The metric $\text{ld}(\mathcal{W}) = \log \left(\prod_j \mu_j(\mathcal{W})\right)$, where $\mu_j(\mathcal{W})$ denotes a non-zero eigenvalue of $\mathcal{W}$, provides a volumetric assessment of the controllable subspace that can be accessed with one unit or less of control energy. 
\end{itemize}

\paragraph*{\textbf{Examples}} We illustrate, through examples, that network controllability is influenced by both the network's topological configuration and the placement of leaders within it. The impact of leader selection on network controllability is visualized in Figure \ref{fig:leader_effect}. In this scenario, we examine a network consisting of 10 agents, with one designated as the leader, resulting in $N_f = 9$. In Figure \ref{fig:leader_effect}(b), the dimension of the controllable subspace is 9, indicating complete controllability of the follower network. The edges connecting the leader and follower nodes, which define the structure of the $B$ matrix in \eqref{eq:follower_dynamics}, are highlighted in red. Transitioning to Figure~\ref{fig:leader_effect}(c), we opt for a different leader while preserving complete controllability; however, this results in a modified trace of $\mathcal{W}$.

\vspace{-0.15in}
\begin{figure}[!ht]
\centering
\begin{subfigure}[b]{0.22\textwidth}
\centering
\includegraphics[scale=0.115]{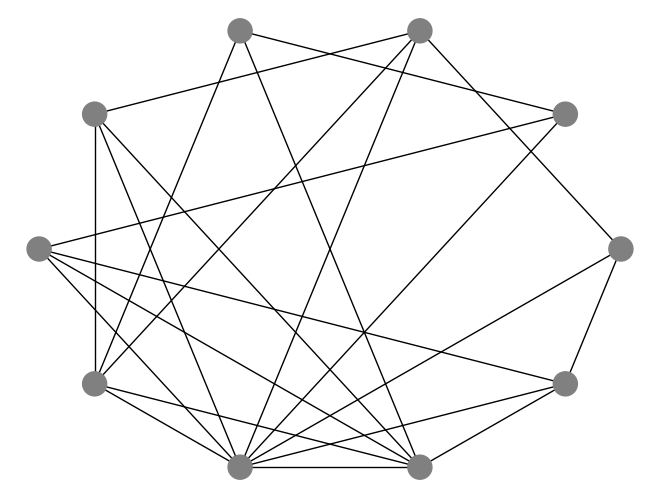}
\caption{input $G$}
\end{subfigure}
\begin{subfigure}[b]{0.22\textwidth}
\centering
\includegraphics[scale=0.115]{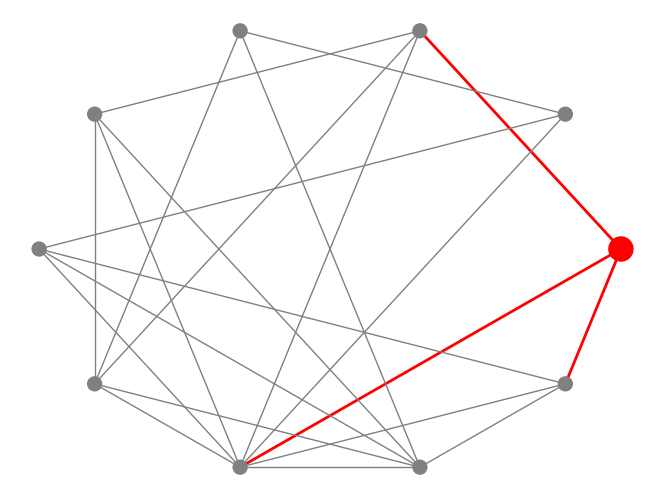}
\caption{rank = 9, tr = 1.5}
\end{subfigure}
\begin{subfigure}[b]{0.22\textwidth}
\centering
\includegraphics[scale=0.115]{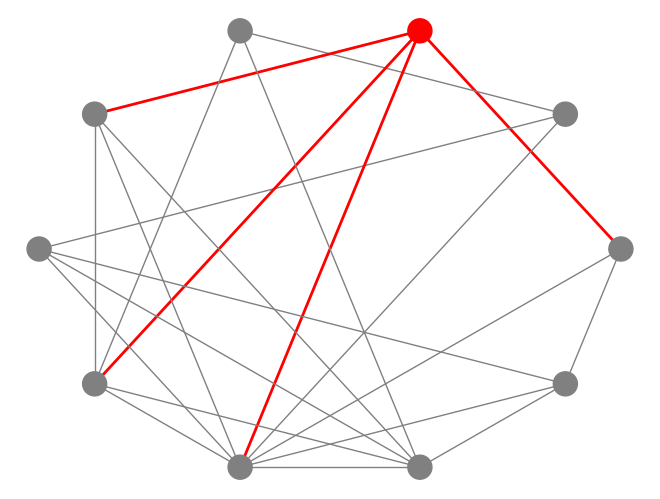}
\caption{rank = 9, tr = 3.5}
\end{subfigure}
\begin{subfigure}[b]{0.22\textwidth}
\centering
\includegraphics[scale=0.115]{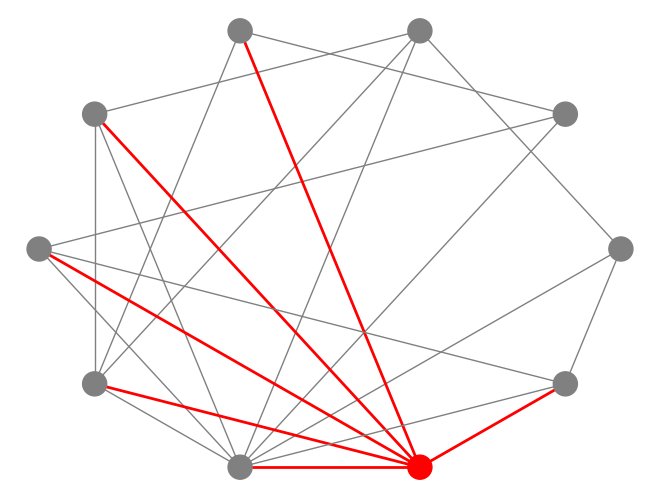}
\caption{rank = 8, tr = 7.1}
\end{subfigure}
\caption{Controllability metrics vary with leader selection}
\label{fig:leader_effect}
\end{figure}
To gather valuable insights into the graph structure, we employ an effective probing strategy. By varying the number and positions of leader nodes, we can observe the resulting controllability behavior, as quantified by metrics such as $\textbf{tr}(\mathcal{W})$, $\mu_j(\mathcal{W})$, $\textbf{rank}(\mathcal{W})$, and $\textbf{ld}(\mathcal{W})$. In our empirical evaluation in Section \ref{sec:emp_eval}, we primarily focus on the rank, trace, and both minimum and non-zero eigenvalues of the Gramian matrix, considering multiple leader set configurations.

\section{Control-Based Graph Augmentations}
\label{sec:augmentation}

Contrastive Graph Representation Learning (CGRL) is a self-supervised technique that relies on augmented data to create positive and negative pairs. As discussed in the previous section, we have developed graph-level control embeddings that act as inputs to the encoder, minimizing the NT-Xent loss~\cite{oord2018representation}. In this section, we will introduce several innovative methods for data augmentation.

The primary purpose of data augmentation is to generate new data that is logically consistent while preserving the semantic labels. As depicted in Figure \ref{fig:block_dia}, CGRL incorporates a contrastive module into the conventional Graph Machine Learning (GML) architecture, introducing a contrastive loss for fine-tuning the models. This process involves contrasting the original graph with an augmentation graph using positive–positive and positive–negative pairs. The core idea behind CGRL is that the original graph should, to some extent, be equivalent to the augmentation graph. Therefore, each graph in the dataset should have precisely one corresponding positive pair (an equivalent graph) in the augmented dataset, while the remaining augmented graphs serve as negative pairs. During training, CGRL strives to optimize the similarities of positive–positive pairs to approach unity and the similarities of positive–negative pairs to approach zero.

Since the graph-level embeddings used in our work are based on control properties, the primary goal is to devise an augmentation technique that primarily preserves the control properties of the original graph in the augmented version.

\vspace{2 mm}
\noindent\fbox{\begin{varwidth}{\dimexpr\linewidth-1\fboxsep-2\fboxrule\relax}
    \begin{probs}
    \label{prob:Aug}
        Given a graph $G = (V, E)$ and a leader set $V_\ell$, perform an augmentation $\calT(G)$ to obtain $G^\prime$ by perturbing $k$ edges while ensuring that the controllability properties are preserved.
    \end{probs}
\end{varwidth}}
\vspace{1 mm}

In this context, \emph{edge perturbation} can refer to actions such as edge deletion, edge addition, or edge substitution.

While performing augmentation, our primary focus is on preserving one of the key features of our CTRL embedding, which is the rank of controllability denoted as $\gamma(G, V_\ell)$. However, preserving the exact rank of controllability is a highly complex task. Consequently, the literature has explored lower bounds on network controllability~\cite{chapman2013strong, yaziciouglu2016graph, monshizadeh2015strong}. In our edge perturbation techniques, we employ a rigorous lower bound based on topological node distances and introduce algorithms to augment the original network while preserving this lower bound~\cite{yaziciouglu2016graph}.

Assuming the presence of $m$ leaders denoted as $V_\ell = \{\ell_1, \ell_2, \cdots , \ell_m\}$ within a leader-follower network $G = (V,E)$, we define the \emph{distance-to-leader (DL) vector} for each vertex $v_i \in V$ as follows:
$$
D_i = \left[\begin{array}{lllll}
    d(\ell_1, v_i) & d(\ell_2, v_i) & \cdots & d(\ell_{m}, v_i)  \\
\end{array}\right]^T \in \mathbb{Z}^m.
$$
In this vector, the $j^{th}$ component, denoted as $[D_i]_j$, represents the distance between leader $\ell_j$ and vertex $v_i$. We then proceed to define a \emph{sequence} of distance-to-leader vectors, referred to as a \emph{pseudo-monotonically increasing sequence} (PMI), as described in \cite{yaziciouglu2016graph}.

\begin{definition} (\emph{Pseudo-monotonically Increasing Sequence (PMI))} A sequence $\calD = [\calD_1 \; \; \calD_2 \;\; \cdots \;\calD_k]$ of distance-to-leader vectors is a PMI if, for any vector $\calD_i$ in the sequence, there exists a coordinate $\pi(i)\in \{1, 2,\cdots, m\}$ such that
\begin{equation}
\label{eq:PMI}
    [\calD_i]_{\pi(i)} < [\calD_j]_{\pi(i)}, \; \forall j > i. 
    \end{equation}
\end{definition} 

In essence, the PMI property \eqref{eq:PMI} ensures that for each vector $\mathcal{D}_i$ in the PMI sequence, there is an index/coordinate $\pi(i)$ such that the values of all subsequent vectors at the coordinate $\pi(i)$ are strictly greater than $[\mathcal{D}_i]_{\pi(i)}$. 

The length of the PMI sequence provides a precise lower bound on the dimension of the controllable subspace $\gamma(G,V_\ell)$. This is presented in the subsequent result.

\begin{theorem}~\cite{yaziciouglu2016graph}
If we denote the length of the longest PMI sequence of DL vectors in a network $G = (V, E)$ with $V_\ell$ leaders as $\delta(G, V_\ell)$ or simply $\delta(G)$, then we can establish the inequality: $ \delta(G, V_\ell) \leq \gamma(G, V_\ell)$
Here, $\gamma(G,V_\ell)$ represents the dimension of the controllable subspace.
\end{theorem}

We propose the following three sophisticatedly designed edge perturbation methods that maintain the lower bound $\delta(G, V_\ell)$ on the rank of controllability. They are illustrated in Figure \ref{fig:augmentation}. The red vertices represent leaders, the gray dashed edge can be removed, and the blue dashed edge can be added while ensuring that the bound $\delta(G, V_\ell) = N_f = 4$ is maintained for all augmented graphs.
\vspace{-0.15in}
\begin{figure}[!ht]
\centering
\begin{subfigure}[b]{0.22\textwidth}
\centering
\includegraphics[scale=0.115]{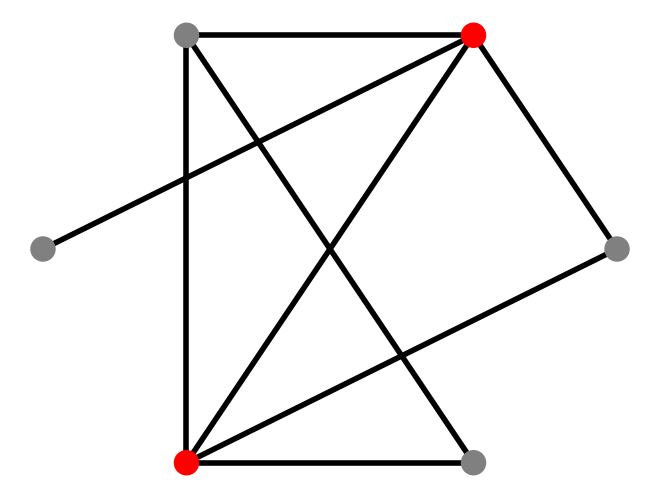}
\caption{input $G$}
\end{subfigure}
\begin{subfigure}[b]{0.22\textwidth}
\centering
\includegraphics[scale=0.115]{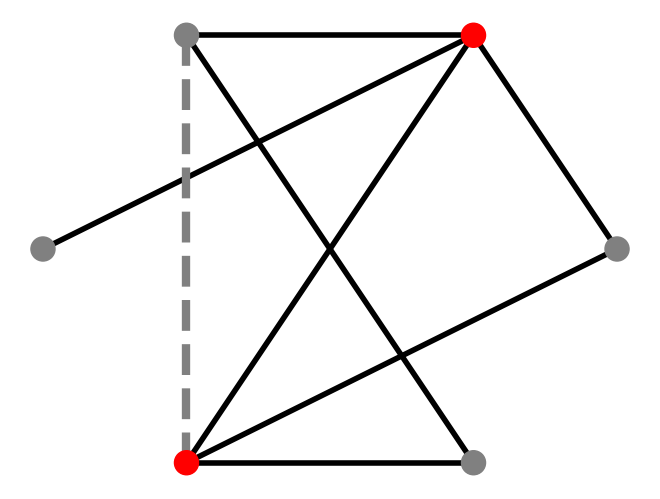}
\caption{Edge deletion}
\end{subfigure}
\begin{subfigure}[b]{0.22\textwidth}
\centering
\includegraphics[scale=0.115]{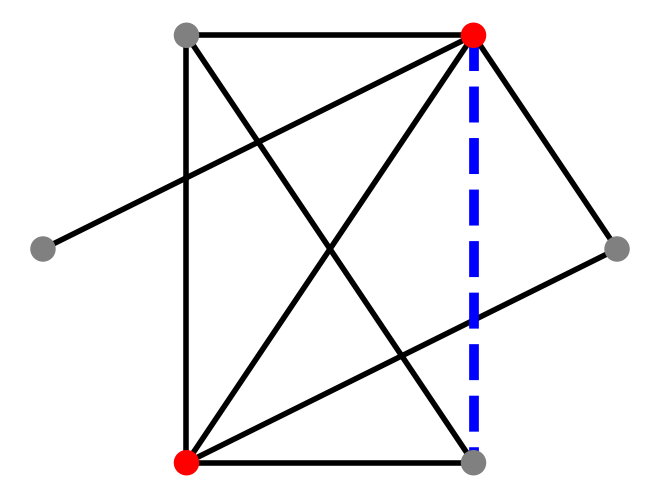}
\caption{Edge addition}
\end{subfigure}
\begin{subfigure}[b]{0.22\textwidth}
\centering
\includegraphics[scale=0.115]{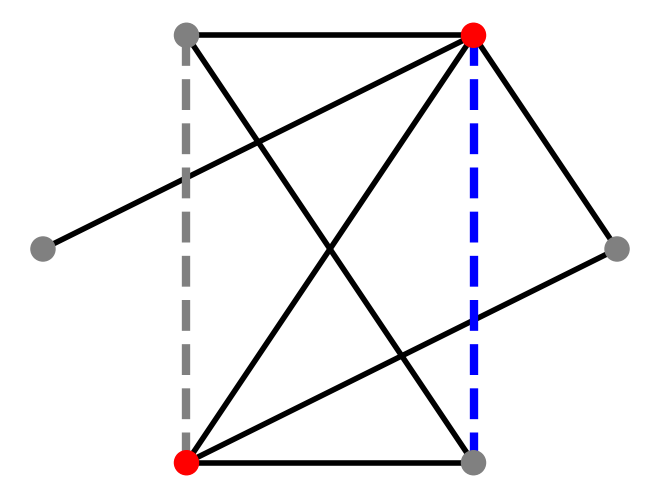}
\caption{Edge substitution}
\end{subfigure}
\caption{Control-based graph augmentations where $\delta = \gamma = 4$ for original and augmented graphs.}
\label{fig:augmentation}
\end{figure}
\subsection{Egde Deletion} 
\label{subsec:edge_del}
We propose using the concept of discerning essential edges, referred to as controllability backbone edges that we introduced in ~\cite{ahmad2023controllability}, which do not decrease $\delta(G, V_\ell)$. We begin by introducing the concept of the distance-based controllability backbone and subsequently present an algorithm for computing an augmented graph with $k$-perturbed edges while utilizing the controllability backbone.

\begin{definition} (\emph{Controllability Backbone})\cite{ahmad2023controllability} For a given graph $G = (V, E)$ and a set of leaders $V_\ell$, the controllability backbone is represented as a subgraph $B = (V, E_{B})$. In this subgraph $B$, the condition $            \delta(G, V_\ell) \leq \delta(\hat{G}, V_\ell).$ holds for any subgraph $\hat{G} = (V, \hat{E})$ where the edge set satisfies $E_{B} \subseteq \hat{E} \subseteq E$.
\end{definition}

In other words, for any subgraph $\hat{G} = (V, \hat{E})$ of $G$ that includes the backbone edges $E_B$, it ensures that at least the same level of controllability as the original graph $G$ is maintained. Essentially, retaining the backbone edges guarantees that controllability remains unchanged or improves within any subgraph $\hat{G}$.

Now, we present Algorithm \ref{alg:edge_deletion} to perform graph augmentation by deleting edges from a graph. First, we compute the important edges (of the controllability backbone) that need to be preserved to maintain the minimum bound on controllability. The controllability backbone can be found by using Algorithm 2 of \cite{ahmad2023controllability}. Then, we calculate a set of potential edges in the given graph $G$ that does not include any edges from the backbone graph $B$. Finally, we randomly select $k$ edges from the set of potential edges and remove them from the original graph $G$ to obtain an augmented graph $G^\prime$. If $k$ is larger than the set of potential edges, we delete all the edges of the potential edge set.

\begin{algorithm}[ht]
\caption{$Edge\_Deletion$}
\label{alg:edge_deletion}
    \begin{algorithmic}[1]
    \renewcommand{\algorithmicrequire}{\textbf{Input:}}
    \renewcommand{\algorithmicensure}{\textbf{Output:}}
    \Require $G = (V, E)$, $V_\ell$, $k$
    \Ensure $G^\prime = (V, E^\prime)$, $|E| - |E^\prime| = k$
        \State Compute the distance-based controllability backbone $B = (V, E_B)$ for $G = (V, E)$ and $V_\ell$.
        \State $pot\_edges \gets E \setminus E_B$  \texttt{\% Set of potential edges}.
        \State $E_{pot} \gets$ randomly selected $k$ edges from $pot\_edges$ 
        \State $E^\prime \gets E\setminus E_{pot}$\\
        \Return $G^\prime = (V, E^\prime)$
    \end{algorithmic}
\end{algorithm}

\begin{prop}
\label{prop:edge_del}
Given a graph $G=(V,E)$ and a leader set $V_\ell$, Algorithm \ref{alg:edge_deletion} returns an augmented graph $G^\prime=(V,E^\prime)$, where $E^\prime \subseteq E$, while ensuring $
    \delta(G, V_\ell) \leq \delta(G^\prime, V_\ell)$.
\end{prop}
\begin{proof}
Let $G = (V, E)$ be a graph with a set of leaders $V_\ell$. The backbone graph $B = (V, E_B)$ is defined as the smallest set of essential edges required to maintain the necessary distances within the graph for establishing the lower bound on controllability, as formally described in Theorem 4.2 of \cite{ahmad2023controllability}.
The theorem establishes that for any edge $e$ present in $G$ but absent in $B$, the removal of edge $e$ does not diminish the lower bound on $\gamma(G, V_\ell)$. Therefore, any edge that is not part of the backbone graph $B$ can be eliminated from $G$ without reducing the lower bound on the rank of the controllability matrix.
Furthermore, this property implies that the removal of any subset of these non-backbone edges from $G$ continues to maintain the bound on the rank of the controllability matrix.
\end{proof}

\subsection{Egde Addition}
\label{subsec:edge_addition}

Building on the concept of recognizing removable edges while preserving the distance-based bound $\delta(G, V_\ell)$, as proposed in our previous work~\cite{abbas2020improving}, we use an augmentation method that determines the edges that can be added to a given graph $G = (V, E)$ while still maintaining the distance-based bound $\delta(G, V_\ell)$. We employ the same approach outlined in \cite{abbas2020improving} to identify edges that can be added to $G$ without diminishing the distance-based bound $\delta(G, V_\ell)$. After determining such potential edges that can be safely incorporated into $G$ without reducing the controllability rank $\gamma(G, V_\ell)$, we randomly select $k$ such potential edges and introduce them into $G$, resulting in the augmented graph $G^\prime = (V, E^\prime)$.


\begin{prop} 
\cite{abbas2020improving}
    If $\delta(G, V_\ell)$ serves as a lower bound for the dimension of the controllable subspace $\gamma(G, V_\ell)$ of a graph \(G = (V, E)\) with leaders \(V_\ell \subset V\), then it also serves as a lower bound for \(\gamma(G^\prime, V_\ell)\) of an augmented graph \(G^\prime = (V, E^\prime)\), where \(E \subseteq E^\prime\) and \(E^\prime\) contains the edges that preserve the distances of DL vectors in the longest PMI sequence of \(G\).
\end{prop}

\subsection{Egde Substitution}
Next, we propose a novel approach that combines the edge deletion and edge addition methods while preserving the size of the edge set $|E|$ of the given graph $G = (V, E)$ and the bound $\gamma(G, V_\ell)$. Algorithm \ref{alg:edge_substitution} outlines this approach. First, we remove $k$ edges from $G$ to create $\bar{G} = (V, \bar{E})$ using Algorithm \ref{alg:edge_deletion}. Then, we introduce $k$ distinct edges into $\bar{G}$ using the method described in \ref{subsec:edge_addition}, resulting in $G^\prime = (V, E^\prime)$, where $|E| = |E^\prime|$.

\begin{algorithm}[ht]
\caption{$Edge\_Substitution$}
\label{alg:edge_substitution}
    \begin{algorithmic}[1]
    \renewcommand{\algorithmicrequire}{\textbf{Input:}}
    \renewcommand{\algorithmicensure}{\textbf{Output:}}
    \Require $G = (V, E)$, $V_\ell$, $k$
    \Ensure $G^\prime = (V, E^\prime)$
   \State $\bar{G} = (V, \bar{E}) \gets$ $Edge\_Deletion(G, V_\ell, k)$
        \State Compute the maximal edge set $E_{max}$ for ${G} = (V, {E})$ and $V_\ell$.
        \State $pot\_edges$ $\gets$ $E_{max} - {E}$  \% \texttt{Set of potential edges}.
        \State $E_{pot} \gets$ randomly selected $k$ edges from $pot\_edges$ 
        \State $E^\prime \gets \bar{E} \cup E^\prime$.\\
        \Return $G^\prime = (V, E^\prime)$  \end{algorithmic}
\end{algorithm}

The maximal edge set $E_{max}$ can be computed by from Algorithm 1 of our previous work \cite{abbas2020improving}.

\begin{prop}
Given a graph $G=(V,E)$ and a leader set $V_\ell$, Algorithm \ref{alg:edge_substitution} yields an augmented graph $G^\prime=(V,E^\prime)$, where $|E^\prime| = |E|$, ensuring that $\delta(G, V_\ell) \leq \delta(G^\prime, V_\ell)$,
\end{prop}
\begin{proof}    
Let $\calD$ be the longest PMI sequence of length $\delta(G, V_\ell)$ in $G = (V, E)$ and $V_\calD \subseteq V$ are the nodes whose DL vectors are in $\calD$. If we remove $k$ edges from $G$ to form $\bar{G} = (V, \bar{E})$, then by Proposition \ref{prop:edge_del}, $\delta(\bar{G}, V_\ell) \leq \delta(G, V_\ell)$ i.e. the longest PMI sequence of $G$ is a subsequence of the longest PMI sequence of $\bar{G}$ as the distances between leaders and
nodes in $V_\calD$ are exactly the same in $G$ and $\bar{G}$.

Now, from Algorithm 1 of \cite{abbas2020improving}, we find a maximal set of edges $E_{max}$ for the original graph $G = (V, E)$. This edge set contains all the edges that can be added to $G$ while preserving the distances between leaders and nodes in $V_\calD$. 
Hence, any edge added to $G$ from $E_{max}$ will maintain the lower bound $\delta(G, V_\ell)$ for $\gamma(G, V_\ell)$. Hence, by deleting $k$ edges that are mutually exclusive from the controllability backbone and adding $k$ edges that keep the distances between leaders and nodes in $V_\calD$ the same, we can substitute $k$ edges in $G$ for given $V_\ell$ such that $\delta(G, V_\ell) \leq \delta(G^\prime, V_\ell)$.
\end{proof}

These edge perturbation methods are employed to create positive pairs. Subsequently, we use the CTRL embeddings to generate graph representations for these pairs and apply the NT\_Xent loss for unsupervised learning of representations for each graph within the dataset. In the following section, we conduct an empirical assessment using real-world graph datasets. Our proposed approach is numerically evaluated through the task of graph classification and compared with state-of-the-art methods.

\section{Numerical Evaluation}
\label{sec:emp_eval}

\subsection{Benchmark Datasets}
{\color{black}
\textbf{Datasets:} We conducted experiments on 7 standard graph classification benchmark datasets, which include MUTAG, PTC\_MR, PROTEINS, and DD, representing bioinformatics datasets, as well as IMDB-BINARY, IMDB-MULTI, and COLLAB, representing social network datasets \cite{Morris+2020}. The bioinformatics datasets provide descriptions of small molecules and chemical compounds. Among the social network datasets, IMDB-BINARY and IMDB-MULTI describe actors' ego-networks, while COLLAB is a scientific collaboration dataset where graphs consist of researchers as nodes and their collaborations as edges. Basic dataset statistics are provided in Table \ref{tab:dataset-stats}.

\begin{table}[!t]
\centering
\caption{Statistics of the datasets. Number of graphs, average number of nodes and edges, range of number of vertices, and the number of classes.}
\label{tab:dataset-stats}
\resizebox{\columnwidth}{!}{%
\begin{tabular}{l|ccccc}
\hline
\textbf{Dataset} & \textbf{\#Graphs} & $\mathbf{avg.|V|}$ & $\mathbf{avg.|E|}$ & $ \mathbf{Range (|V|)} $ & \textbf{\#Classes} \\ \hline
MUTAG & 188 & 17.93 & 19.79 & 10-28 & 2  \\ 
PTC & 344 & 14.29 & 14.69 & 2-109 & 2  \\ 
PROTEINS & 1113 & 39.06 & 72.82 & 4-620 & 2  \\ 
DD & 1178 & 284.32 & 715.66 & 30-743 & 2 \\ 
COLLAB & 5000 & 284.32 & 715.66 & 32-492 & 3 \\ 
IMDB-B & 1000 & 19.77 & 96.53 & 12-136 & 2 \\ 
IMDB-M & 1500 & 13.00 & 65.94 & 7-89 & 3   \\ \hline
\end{tabular}
}
\vspace{-0.2in}
\end{table}
}

\subsection{The Role of Data Augmentation in Graph CL}
We evaluate the effectiveness of our proposed framework for graph classification using the TUDataset benchmark~\cite{Morris+2020}. We utilize the CL method to unsupervisedly learn representations $z(.)$ from CTRL embeddings, followed by the evaluation of these representations for graph-level classification. This evaluation involves training and testing a linear SVM classifier using the acquired representations. We employ a 10\% label rate and 10-fold cross-validation, conducting experiments over 5 repetitions and reporting the evaluation accuracy as a mean value along with the standard deviation.

As a baseline reference, we directly employ the CTRL embeddings for training the SVM classifier. The results for four distinct bioinformatics datasets are summarized in Table \ref{tab:augmentations}. Edge deletion yields the best result on MUTAG and PTC, while on PROTEINS and DD, edge addition and substitution provide the best results, respectively. We vary the number of edges perturbed from 1 to 3. Our augmentation techniques, combined with contrastive learning, consistently yield higher classification accuracies across all datasets compared to the baseline approach.
\begin{table}[ht]
\centering
\caption{Graph Classification accuracy (\%) with 10\% label rate. The baseline involves a linear SVM trained directly on CTRL embeddings.}
\label{tab:augmentations}
\resizebox{\columnwidth}{!}{%
\begin{tabular}{p{1.5cm}l|lcccc} 
\hline
\multicolumn{2}{l|}{\textbf{Method}}  & \textbf{MUTAG}  & \textbf{PTC} & \textbf{PROTEINS} & \textbf{DD} \\ 
\hline
\multicolumn{2}{l|}{Baseline}   & 75.86 $\pm$ 11.0 & 52.85 $\pm$ 9.5 & 58.72 $\pm$ 11.9 & 59.10 $\pm$ 13.9 \\
\multicolumn{2}{l|}{CGCL} & 79.54 $\pm$ 11.0 & 56.10 $\pm$ 8.3 & 69.97 $\pm$ 4.5 & 63.22 $\pm$ 11.1 \\
\hline
\end{tabular}
}
\end{table}
\subsection{ Comparison with the State-of-the-art Methods}
The effectiveness of CGCL is assessed in the context of unsupervised representation learning, following the approach outlined in \cite{sun2019infograph, narayanan2017graph2vec}. We closely adhere to the established approach within Graph CL for graph classification~\cite{sun2019infograph, you2020graph}. We compare our results with the following kernel-based, unsupervised, and self-supervised methods.

\paragraph{Graph Kernels Methods} Graph kernel-based techniques represent traditional approaches to graph classification. They engage directly with graph data by crafting kernel functions that preserve the graph's structural information. For the purpose of this comparative analysis, we have chosen four well-established graph kernel-based methods: Graphlet kernel (GK) \cite{shervashidze2009efficient}, Weisfeiler-Lehman sub-tree kernel (WL) \cite{shervashidze2011weisfeiler}, and deep graph kernels (DGK) \cite{yanardag2015deep}.

\paragraph{Unsupervised Methods} Unsupervised techniques for graph representation learning leverage sub-graph and node similarity scores to guide the learning process, all without relying on label information. In this comparative analysis, we have chosen three prominent unsupervised benchmarks methods: node2vec \cite{grover2016node2vec}, sub2vec \cite{adhikari2018sub2vec}, and graph2vec \cite{narayanan2017graph2vec}.

\paragraph{Self-Supervised Methods} Unlike traditional unsupervised graph representation techniques, the self-supervised approach to graph representation harnesses the capabilities of contrastive learning to unearth profound insights between data samples, leading to significant improvements in graph representation learning. In this comparative study, we've chosen two state-of-the-art contenders, InfoGraph \cite{sun2019infograph} and GraphCL \cite{you2020graph}, which both employ graph neural networks as their foundational architecture.
 
\begin{table*}[t]
\tiny
\caption{Comparing classification accuracy on top of graph representations learned from graph kernels, SOTA representation learning method. The top two results are highlighted by \textcolor{red}{First}, \textcolor{blue}{Second}. The numerical values presented for comparison are obtained from the respective papers, following the identical experimental configurations.}
\label{tab:results}
\resizebox{\textwidth}{!}{%
\begin{tabular}{lccccccc} 
\hline
\hline
Methods       & MUTAG             & PTC              & PROTEINS         & DD               & COLLAB           & IMDB-B           & IMDB-M \\
\hline
\hline
\multicolumn{8}{c}{\textbf{Kernel Approaches}}\\
GK            &   81.70 $\pm$ 2.1 & 57.30 $\pm$ 1.4  & -   & -  & 72.80 $\pm$ 0.3  & 65.90 $\pm$ 1.0  & 43.90 $\pm$ 0.4  \\
WL            & 80.63 $\pm$ 3.1  & 56.91 $\pm$ 2.8 & 72.92 $\pm$ 0.6 & -                & \textcolor{red}{78.90 $\pm$ 1.9}  & 72.30 $\pm$ 3.4 & 47.00 $\pm$ 0.5        \\
DGK           & 87.44 $\pm$ 2.7  &   60.10 $\pm$ 2.6 & 73.30 $\pm$ 0.8 & -                & -                & 66.96 $\pm$ 0.6 & 44.60 $\pm$ 0.5 \\ 
\hline
\hline
\multicolumn{8}{c}{\textbf{Unsupervised Approaches}}\\
node2vec      & 72.63 $\pm$ 10.2 & 58.85 $\pm$ 8.0 & 57.48 $\pm$ 3.6 & -                & 55.70 $\pm$ 0.2
& 50.20 $\pm$ 0.9  & 36.0 $\pm$ 0.7         \\
sub2vec       & 61.05 $\pm$ 15.8 & 59.99 $\pm$ 6.4 & 53.03 $\pm$ 5.6 & -                & 62.10 $\pm$ 1.4  & 55.26 $\pm$ 1.5 &   36.7 $\pm$ 0.8     \\
graph2vec     & 83.15 $\pm$ 9.2  & 60.17 $\pm$ 6.9 & 73.30 $\pm$ 2.1 & -                & 59.90 $\pm$ 0.0 & 71.10 $\pm$ 0.5 &  \textcolor{red}{50.40 $\pm$ 0.9}      \\
\hline
\hline
\multicolumn{8}{c}{\textbf{Self-Supervised Approaches}}\\
InfoGraph & \textcolor{blue}{89.01 $\pm$ 1.1}  &  \textcolor{blue}{61.70 $\pm$ 1.4} & \textcolor{red}{74.44 $\pm$ 0.3} & 72.85 $\pm$ 1.8 & 70.65 $\pm$ 1.1 & \textcolor{red}{73.03 $\pm$ 0.9} & \textcolor{blue}{49.70 $\pm$ 0.5} \\
GraphCL       & 86.80 $\pm$ 1.3  &  61.30 $ \pm$ 2.1 & \textcolor{blue}{74.39 $\pm$ 0.5} & \textcolor{red}{78.62 $\pm$ 0.4} & 71.36 $\pm$ 1.2 & 71.14 $\pm$ 0.4 &  48.58 $\pm$ 0.7      \\
\hline
\textbf{CGCL} & \textcolor{red}{89.91 $\pm$ 6.4}  & \textcolor{red}{66.34  $\pm$ 7.9} &  74.13 $\pm$ 2.8 &  \textcolor{blue}{75.33 $\pm$ 3.3} & \textcolor{blue}{75.10 $\pm$ 1.8} & \textcolor{blue}{72.70 $\pm$ 4.4} &   48.53 $\pm$ 3.1    \\
\hline
\hline
\end{tabular}
}
\end{table*}

The results of graph classification are presented in Table \ref{tab:results}. When compared to the top-performing unsupervised methods, our proposed approach exhibits significant improvements across several datasets, including MUTAG, PTC, PROTEINS, COLLAB, and IMDB-B, resulting in gains of 6.76\%, 6.17\%, 0.83\%, 13.0\%, and 1.6\%, respectively. Notably, our method consistently outperforms all unsupervised competitors across all datasets except IMDB-M.

In contrast to self-supervised counterparts, our proposed method surpasses the current SOTA on MUTAG, PROTEINS, and COLLAB, and achieves the second-best accuracies on DD and IMDB-B datasets. Specifically, in MUTAG, PROTEINS, and COLLAB, our approach outperforms the existing standards by margins of 0.90\%, 4.64\%, and 3.74\%, respectively. However, on PROTEINS and IMDB-M, our method falls short by 0.31\% and 1.17\%, respectively, compared to the best self-supervised approaches.

In summary, as demonstrated by Table \ref{tab:results}, our proposed method outperforms its competitors on two out of seven graph classification datasets by a considerable margin and achieves top-two accuracy rankings for five out of seven datasets. For the remaining two datasets, our method achieves accuracy levels within 2\% of the SOTA methods.

We also perform an evaluation of our proposed CGCL approach using the edge augmentation method proposed by You~\cite{you2020graph}. We follow an i.i.d. uniform distribution to add/drop edges instead of the systematic edge perturbation methods mentioned in Section \ref{sec:augmentation}. We call this approach Random-CGCL. Table \ref{tab:aug - results} presents the results for graph classification accuracies for all seven datasets under consideration. It can be seen that the accuracy of Random-CGCL is comparable to both GraphCL and CGCL. However, CGCL outperforms Random-CGCL on all the datasets. These results suggest that a sophisticated augmentation technique, as employed in CGCL, is essential for effectively leveraging control-based embeddings in graph contrastive learning.

\begin{table*}[t]\tiny
\caption{Graph Classification accuracies using different Augmentation methods. The top accuracies are \textbf{highlighted}.}
\label{tab:aug - results}
\resizebox{\textwidth}{!}{%
\begin{tabular}{lccccccc} 
\hline
\hline
Methods       & MUTAG             & PTC              & PROTEINS         & DD               & COLLAB           & IMDB-B           & IMDB-M \\
\hline
\hline
GraphCL       & 86.80 $\pm$ 1.3  &  61.30 $ \pm$ 2.1 & \textbf{74.39 $\pm$ 0.5} & \textbf{78.62 $\pm$ 0.4} & 71.36 $\pm$ 1.2 & 71.14 $\pm$ 0.4 &  \textbf{48.58 $\pm$ 0.7} \\
\hline
Random-CGCL & 87.81  $\pm$ 7.4  & 65.73 $\pm$ 6.6 & 73.23 $\pm$ 1.8 &  75.15 $\pm$ 2.9 & 75.04 $\pm$ 1.8 & 71.40 $\pm$ 4.0 & 47.40  $\pm$ 2.7 \\
\hline
\textbf{CGCL} & \textbf{89.91 $\pm$ 6.4}  & \textbf{66.34  $\pm$ 7.9} &  74.13 $\pm$ 2.8 &  75.33 $\pm$ 3.2 & \textbf{75.10 $\pm$ 1.8} & \textbf{72.70 $\pm$ 4.4} &   48.53 $\pm$ 3.1    \\
\hline
\end{tabular}
}
\end{table*}

\section{Conclusion}
\label{sec:conclusion}
In this work, we introduced Control-based Graph Contrastive Learning (CGCL), a novel framework for unsupervised graph representation learning that leverages graph controllability properties. We employed advanced edge augmentation methods to create augmented data for contrastive learning while preserving the controllability rank of graphs. Our extensive experiments on standard graph classification benchmarks showcased CGCL's effectiveness, outperforming SOTA unsupervised and self-supervised methods on multiple datasets. We also compared CGCL with a random edge augmentation approach, underscoring the significance of our controllability-driven augmentation strategy. The success of CGCL suggests that incorporating domain-specific structural knowledge, like controllability, can significantly enhance graph representation learning, opening avenues for further research. Overall, CGCL presents a promising approach for applications requiring informative graph representations. In the future, we aim to refine our graph augmentation techniques to preserve all relevant control features.

\bibliographystyle{IEEEtran}
\vspace{-0.08in}
\bibliography{References}
\end{document}